%% file: neurips_2024.tex
\documentclass{article} 
\usepackage{fullpage,appendix}
\PassOptionsToPackage{numbers, compress}{natbib}
\input{math_commands.tex}

\usepackage{natbib}

\usepackage[utf8]{inputenc} 
\usepackage[T1]{fontenc}    
\usepackage{hyperref}       
\usepackage{url,enumitem}            
\usepackage{booktabs}       
\usepackage{amsfonts}       
\usepackage{nicefrac}       
\usepackage{microtype}      
\usepackage{xcolor}         

\usepackage{amsmath,amsthm,amssymb}
\usepackage{graphicx}
\usepackage{wrapfig}
\usepackage{bm}
\usepackage{subfig}
\usepackage{array}
\usepackage{multirow}
\usepackage{makecell}
\usepackage{color, colortbl}
\usepackage{cleveref}
\usepackage{enumitem}
\usepackage{pifont}
\usepackage{algorithm}
\usepackage{algorithmic}

\newtheorem{theorem}{Theorem}

\newtheorem{definition}{Definition}

\title{MoreauPruner: Robust Pruning of Large Language Models against Weight Perturbations}

\author{
Zixiao~Wang,
Jingwei~Zhang,
Wenqian~Zhao,
Farzan~Farnia,
Bei~Yu,
\\
The Chinese University of Hong Kong\\
\texttt{\{zxwang22, jwzhang22, wqzhao, farnia, byu\}@cse.cuhk.edu.hk} \\
} 

\date{}

\usepackage{comment}

\newcommand{\minisection}[1]{\vspace{.1in}\noindent{\textbf{#1}}}

\begin{document}

\maketitle
 
\input{docs/abs.tex}
\input{docs/intro.tex}

\input{docs/related.tex}
\input{docs/pre.tex}

\input{docs/method.tex}

\input{docs/exp.tex}

\input{docs/conclusion.tex}

\bibliography{neurips_2024}
\bibliographystyle{unsrt}
\appendix
\input{docs/proof}
\input{docs/appendix.tex}

\end{document}

%% file: math_commands.tex
\usepackage{amsmath,amsfonts,bm}

\def\eqref#1{equation~\ref{#1}}

\def\1{\bm{1}}

\def\rvv{{\mathbf{v}}}

\def\rvx{{\mathbf{x}}}

\def\rmH{{\mathbf{H}}}

\def\rmZ{{\mathbf{Z}}}

\def\vc{{\bm{c}}}

\def\vg{{\bm{g}}}

\def\vu{{\bm{u}}}
\def\vv{{\bm{v}}}
\def\vw{{\bm{w}}}
\def\vx{{\bm{x}}}

\def\vz{{\bm{z}}}

\DeclareMathAlphabet{\mathsfit}{\encodingdefault}{\sfdefault}{m}{sl}
\SetMathAlphabet{\mathsfit}{bold}{\encodingdefault}{\sfdefault}{bx}{n}

\def\gG{{\mathcal{G}}}

\def\gO{{\mathcal{O}}}

\DeclareMathOperator*{\argmin}{arg\,min}

%% file: docs/abs.tex
\begin{abstract}
     Few-shot gradient methods have been extensively utilized in existing model pruning methods, where the model weights are regarded as static values and the effects of potential weight perturbations are not considered. However, the widely used large language models (LLMs) have several billion model parameters, which could increase the fragility of few-shot gradient pruning. In this work, we experimentally show that one-shot gradient pruning algorithms could lead to unstable results under perturbations to model weights. And the minor error of switching between data formats bfloat16 and float16 could result in drastically different outcomes. To address such instabilities, we leverage optimization analysis and propose an LLM structural pruning method, called MoreauPruner, with provable robustness against weight perturbations. In MoreauPruner, the model weight importance is estimated based on the neural network's Moreau envelope, which can be flexibly combined with $\ell_1$-norm regularization techniques to induce the sparsity required in the pruning task. We extensively evaluate the MoreauPruner algorithm on several well-known LLMs, including LLaMA-7B, LLaMA-13B, LLaMA3-8B, and Vicuna-7B. Our numerical results suggest the robustness of MoreauPruner against weight perturbations, and indicate the MoreauPruner's successful accuracy-based scores in comparison to several existing pruning methods. We have released the code in \url{https://github.com/ShiningSord/MoreauPruner}.
\end{abstract}

%% file: docs/intro.tex
\section{Introduction}
\label{sec:intro}

In the rapidly evolving field of Natural Language Processing (NLP), transformer-based Large Language Models (LLMs) such as GPTs \cite{dettmers2022gpt3} and LLaMAs \cite{touvron2023llama,llama3modelcard} have become foundational technologies, driving significant advances across various tasks. These models excel in understanding and generating human language due to their scalable architecture, which allows performance to improve with an increase in parameters. However, deploying these large models poses significant challenges due to their substantial computational and memory demands. To address these challenges, considerable research has been directed toward model pruning \cite{han2015learning,wen2016learning,ma2023llmpruner,zhang2023pruning}, a technique aimed at reducing model size while maintaining or enhancing model performance.

While effective in accelerating LLMs for efficient deployment, existing pruning methods generally treat the parameters of a pre-trained language model as fixed, neglecting potential perturbations in the weights. These perturbations can originate from various sources, including quantization errors during transitions between precision levels and alterations from parameter-efficient fine-tuning (PEFT) technologies that modify a small subset of parameters. 
When there is a need to prune those slightly modified models, we may expect that the pruning results getting from modified models similar to those of the original models. For example, existing LLMs are usually trained with the weight format bfloat16 (BF16) and deployed with the weight format float16 (FP16). As both BF16 and FP16 utilize 16-bit to represent a floating point, the negligible transition error will not affect inference results in most cases. Considering that the basic idea of pruning is removing unnecessary weights and keeping the essential weights, it is straightforward to believe that the models pruned from BF16 and FP16 will be close to each other. However, current gradient-dependent pruning methods \cite{ma2023llmpruner,zhang2023pruning,lecun1989optimal,hassibi1992second} utilize gradient to indicate the importance of weight elements while gradient is known to be sensitive to such weight perturbations, leading to significant variations in pruning outcomes, as depicted in \Cref{fig:motivation}. Such inconsistency in pruned outcomes could be anti-intuitive, and a robust pruning algorithm against weight perturbation may further enhance the performance of the compressed models.  

\begin{figure}[tb!]
    \centering
    \includegraphics[width=0.92\linewidth]{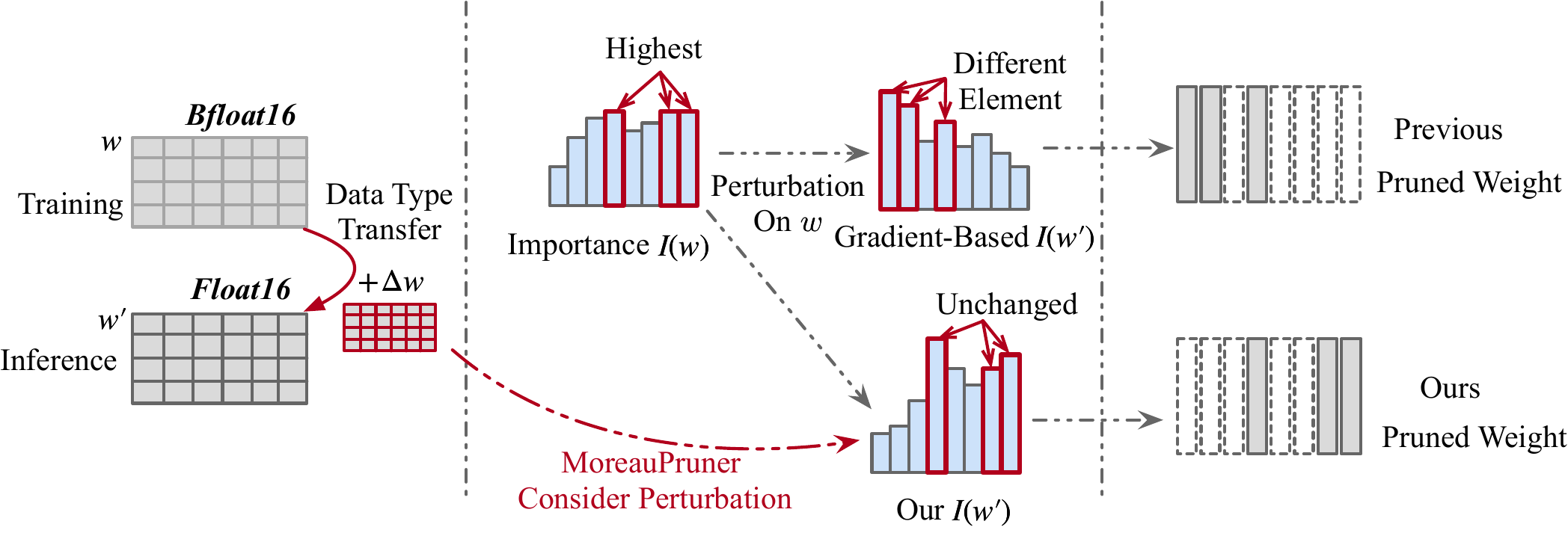}
    \caption{While gradient-based pruning methods are sensitive to weight perturbation, the proposed MoreauPruner gives a robust estimation of weight importance.}
    \label{fig:motivation}
\end{figure}

This paper introduces MoreauPruner, a novel robust structural pruning algorithm for LLMs, designed to mitigate the effects of weight perturbations while preserving model performance. MoreauPruner utilizes the gradient of the loss function's Moreau envelope \cite{moreau1965proximite,  zhang2023moreaugrad,t2020personalized}, a well-established optimization tool for function smoothing, to reduce weight sensitivity to perturbations during the pruning process. We show that the gradient of the Moreau envelope remains stable within the neighborhood of given weight in parameter space. This stability enables MoreauPruner to generate robustness pruning result against weight perturbations, with any norm-bounded perturbation resulting in only a bounded change of the gradient. Additionally, by incorporating an $\ell_1$-group-norm-based regularization penalty, MoreauPruner promotes group-level sparsity in the gradient, which is suitable for structural pruning to facilitate real-life acceleration on hardware platforms. Our empirical results suggest that MoreauPruner improves the robustness of pruning outcomes against weight perturbations and achieves state-of-the-art post-pruning model performance among baseline LLM pruning methods.

Our contributions through this work are threefold:
\begin{itemize}[leftmargin=*]
    \item We emphasize the importance of considering weight perturbations in pruning algorithms, an aspect previously neglected in the literature. This work is among the first to tackle the robustness of pruning algorithms to such perturbations.
    \item We introduce MoreauPruner, a structural pruning algorithm that offers provable robustness to weight perturbations, leveraging the Moreau envelope to ensure the smoothness and stability of the pruning process.
    \item Through extensive experimentation with widely-used large language models such as LLaMA-7B, LLaMA-13B, LLaMA3-8B, and Vicuna-7B, we demonstrate that MoreauPruner achieves a state-of-the-art performance in both robustness to weight perturbation and overall performance of compressed models.
\end{itemize}

%% file: docs/related.tex
\section{Related Work}

\subsection{Efficient Large Language Models}

Large Language Models (LLMs) \cite{touvron2023llama,achiam2023gpt,vicuna2023,llama3modelcard} have achieved remarkable performance by following the scaling laws \cite{kaplan2020scaling}. However, deploying LLMs can be challenging due to high inference costs in resource-limited scenarios. Various methods have been proposed to reduce model size, including knowledge distillation \cite{hinton2015distilling,sanh2019distilbert, sun2019patient, sun2020contrastive}, which involves transferring the knowledge from the original model to a smaller one; model quantization \cite{dettmers2022gpt3, xiao2023smoothquant,yao2022zeroquant,zafrir2019q8bert}, which reduces the bit length of the model weights; and model pruning \cite{han2015learning,frankle2018lottery,fang2023depgraph,park2023accurate}, which involves removing non-essential weights to speed up inference. This work primarily focuses on pruning LLMs \cite{xia2023sheared,ma2023llmpruner,bair2023adaptive,xu2024random,ashkboos2024slicegpt}, where gradients are particularly sensitive to weight perturbations due to the large scale of the model.

\subsection{Pruning Criteria}

To determine which weights to prune during the pruning phase, the importance of each weight is assessed using various criteria. Several studies \cite{sun2023simple,li2018optimization,han2015learning, elesedy2020lottery} adopt a magnitude-based criterion, retaining weights with larger magnitudes post-pruning. Recent approaches \cite{sun2023simple} also consider activation values to evaluate weight importance. Some prevalent criteria are based on Taylor Expansion approximation \cite{lecun1989optimal,ma2023llmpruner,yu2022width,hassibi1993optimal, hassibi1992second}, utilizing differential information (zero-th, first, and second order) to estimate output changes if weights are removed. Notably, \cite{zhang2023pruning} highlights that, in LLMs, gradients can be efficiently approximated using low-rank methods \cite{hu2021lora} when the direct computation of the gradient is too costly. Nevertheless, gradients can be highly sensitive to weight modifications, rendering gradient-based pruning criteria susceptible to variations in weight. In response, MoreauPruner offers proven robustness against any norm-bounded weight perturbation, maintaining model performance.

%% file: docs/pre.tex
\section{Preliminaries}
In this section, we provide a review of prior Taylor-expansion-based structural pruning methods, along with the notation and definitions used in the paper.

\subsection{Notation and Definitions}

Let $\mathcal{D} = \{\vx_i\}_{i=1}^N$ denotes a text dataset with N samples.
$f(\vw,\vx)$ is the next token prediction loss on sample $\vx$ with a parameterized language model and its weight $\vw \in \mathbb{R}^d$. Then the expectation of loss on dataset $\mathcal{D}$ is \begin{equation}
    L(\vw, \mathcal{D}) = \frac{1}{N}\sum_{i=1}^Nf(\vw,\vx_i).
    \label{eq:loss}
\end{equation}
The $\ell_p$-norm of an input vector $\rvv$ is represented as $\|\rvv\|_p$. Furthermore, we use notation $\|\rvv\|_{p,q}$ to denote the $\ell_{p,q}$-group-norm of $\rvv$ deﬁned in the following equation for given variable subsets $S_1,...,S_t\subset\{1,...,d\}$:
\begin{equation}
    \|\rvv\|_{p,q} = \bigl\|\left[\|\rvv_{S_{1}}\|_{p},...,\|\rvv_{S_{t}}\|_{p}\right]\bigr\|_{q},
\end{equation}
which means $\|\rvv\|_{p,q}$ is the $\ell_q$-norm of a vector containing the $\ell_p$-norms of the subvectors of $\rvv$ characterized by index subsets $S_1,...,S_t\subset\{1,...,d\}$.

\subsection{Estimating Importance Score via Taylor Expansion}

Recent pruning method usually estimate the impact of removing a weight element $ w^{(k)}$ via Taylor Expansion,
\begin{equation}
\begin{aligned}
I( w^{(k)}) &= \|L(\vw,\mathcal{D}) - L( w^{(k)}=0,\mathcal{D})\| \\
&=|\frac{\partial L(\vw,\mathcal{D})}{\partial w^{(k)}} w^{(k)} - \frac{1}{2} w^{(k)}\rmH_{kk}(\rvx) w^{(k)} + \gO(| w^{(k)}|^3)|. \nonumber
\end{aligned}
\end{equation}
In the above equations, $L( w^{(k)}=0,\mathcal{D})$ denotes masking out the parameter $\vw$ in the neural network. Hessian matrix $\rmH$ is approximated by a diagonal one and $\rmH_{kk}$ is the $k$-th element on the diagonal. In some of previous pruning methods, the first-term is typically neglected since the the model is well-trained and converged on the training dataset, where $\frac{\partial L(\vw,\mathcal{D})}{\partial w^{(k)}} \approx 0$. However, a recent LLM pruning work\cite{ma2023llmpruner} point out the calibration dataset used in pruning is out of the training data and $\frac{\partial L(\vw,\mathcal{D})}{\partial w^{(k)}} \neq 0$. Given that the heavy computational cost of the Hessian matrix is unacceptable for LLMs, unlike small models, the importance score of parameter $w^{(k)}$ can be approximated using the first term in the Taylor Expansion,
\begin{equation}
    I( w^{(k)}) = |\frac{\partial L(\vw,\mathcal{D})}{\partial w^{(k)}} w^{(k)}|.
    \label{eq:is}
\end{equation}

To achieve real-time acceleration on hardware, structural pruning algorithms remove weight elements in group, \textit{i.e.}, all elements in a channel, blocks or heads. The importance score of such a structure can be easy obtained by summarizing the importance of elements,
\begin{equation}
    I(\vw_i) = \sum_kI(\vw_i^{(k)}),
    \label{eq:structure}
\end{equation}
where $\vw_i$ is the weight vector in a structure.

Existing pruning algorithms tend to pruning those weights with smaller importance scores. However, as we mentioned in \Cref{sec:intro}, the simple gradient in \Cref{eq:is} is sensitive to weight perturbations, which further leads to an unstable pruning result. Motivated by this fact, we have designed a robust pruning criterion in MoreauPruner and detailed it in the next section.

\subsection{Dependency-aware Structural Pruning}

Previous works\cite{fang2023depgraph,ma2023llmpruner} suggest that structural pruning should consider dependency among structures. Here, a weight group $\gG=\{\vw_i\}^M_{i=1}$ represents a collection of coupled structures, where $M$ is the number of structures in one group, and $\vw_i$ denotes the weight for each structure. The group can be effiently detected by \cite{fang2023depgraph}. And the importance score of the group $\gG$ is then estimated as follows:
\begin{equation}
I(\gG) = \mathop{Agg}\limits_{i=1}^{M}I(\vw_i),
\label{eq:group}
\end{equation}
where $\mathop{Agg}$ is a customized aggregation function chosen from options like Summation, Production, Max, etc. After assessing the importance of each group, groups with lower importance are pruned to achieve a pre-determined pruning ratio.

%% file: docs/method.tex
\section{MoreauPruner}

In this section, we introduce the proposed pruning method, MoreauPruner. We start by detailing the proposed perturbation-robust pruning criteria. In the second subsection, we introduce the two versions of MoreauPruner and how they work.

\subsection{Robustifying Gradient via Moreau Envelope}
Here we leverage the notion of Moreau envelope from the convex optimization literature to propose an optimization-based approach to robust gradient-based pruning.
The considered robust gradient follows Moreau-Yosida regularization, based on which the Moreau envelope of a neural network's parameters is defined as follows.
\begin{definition}
    Consider function $g:\mathbb{R}^d \rightarrow \mathbb{R}$ and regularization parameter $\rho>0$. The Moreau envelope of $g$ at input weight $\vw$, $g^\rho:\mathbb{R}^d\rightarrow \mathbb{R}$ is defined to be
    \begin{equation}
        g^\rho(\vw)\, :=\, \inf_{\Tilde{\vw}} g^\rho(\Tilde{\vw}) + \frac{1}{2\rho}\bigl\Vert\Tilde{\vw} - \vw\bigr\Vert^2_2.
    \end{equation}
\end{definition}
Following \cite{zhang2023moreaugrad}, we utilize the gradient of the Moreau envelope as a robust evaluation of the local sensitivity of the loss function sensitivity to altering the model weights.
\begin{definition}
    {Given the input weight $\vw$ and regularization parameter $\rho>0$, we define MoreauGrad as the Moreau envelope $g^\rho$'s gradient $\mathrm{MG}^\rho[g]:\mathbb{R}^d\rightarrow\mathbb{R}^d$:
    \begin{equation}
        \mathrm{MG}^\rho[g](\vw) := \nabla g^\rho(\vw). \label{eq:mg}
    \end{equation}}
\end{definition}

To analyze the gradient of Moreau envelope, we first discuss the optimization-based smoothing enforced by the Moreau envelope. Note that the Moreau envelope is known as an optimization tool to turn non-smooth convex functions (e.g. $\ell_1$-norm) into smooth functions, where the smoothness is usually regarding the input variable $\vx$. Here in the pruning case, we discuss the smoothness regarding the function parameters $\vw$ and extend the result to the weakly-convex functions which also apply to non-convex functions.

\begin{theorem}\label{Thm: MoreauGrad}
    Suppose that the parameterized function $g(\vw):\mathbb{R}^d \rightarrow \mathbb{R}$ is $\beta$-Lipschitz, i.e. it satisfies $\bigl\vert g(\vw) - g(\mathbf{v})\bigr\vert \le \beta\Vert \vw- \mathbf{v} \Vert_2 $ for every $\mathbf{v,w}\in\mathbb{R}^d$.
    Consider the Gaussian-smoothed $g_\sigma(\mathbf{w})= \mathbb{E}_{\mathbf{u}\sim \mathcal{N}(\mathbf{0},\sigma^2 I)}\bigl[g(\vw+\mathbf{u})\bigr]$.
     Then, for every $0<\rho<\frac{\sigma}{\beta}$, the following robustness guarantee will hold of the Moreau envelope of the Gaussian-smoothed $g_\sigma^\rho(\vw)$
        $$\| \mathrm{MG}^\rho[g_\sigma](\vw_1) - \mathrm{MG}^\rho[g_\sigma](\vw_2)\|_2 \leq \frac{\sigma}{\min\{\sigma\rho, \sigma-\rho\beta\}}\|\vw_1-\vw_2\|_2.$$
\end{theorem}

\begin{proof}
    We defer the proof of the above theorem into appendices due to the space limitation.
\end{proof}
We note that, as shown in \cite{zhang2023moreaugrad}, the above definition can be combined with sparsity-based norm penalties, such as $\ell_1$-norm $\Vert \cdot\Vert_1$ or $\ell_{2,1}$-group norm $\Vert \cdot\Vert_{2,1}$ norm. Here, we generalize \cite{zhang2023moreaugrad}'s definition of (group)sparse-Moreau envelop, and given a convex function $h:\mathbb{R}^d\rightarrow \mathbb{R}$ propose the following definition of $h$-Moreau envelope:  

\begin{definition}\label{def: h-moreaugrad}
    {Given convex function $h$, input weight $\vw$, and regularization parameter $\rho>0$, we define $h$-MoreauGrad of function $g$, denoted by $h\text{-}\mathrm{MG}^\rho[g](\mathbf{w})$, as $\frac{1}{\rho}\bigl(\mathbf{v}^*(\mathbf{w})- \mathbf{w}\bigr)$ where $\mathbf{v}^*(\mathbf{w})$ denotes the optimal solution to the following optimization problem:
    \begin{equation}
        \min_{\mathbf{v}\in\mathbb{R}^d}\: g(\mathbf{v}) + \frac{1}{2\rho}\bigl\Vert \mathbf{v} - \mathbf{w} \bigr\Vert_2^2 + h\bigl(\mathbf{v}- \mathbf{w}\bigr).
    \end{equation}}
\end{definition}
In our numerical analysis, we specifically focus on GroupSparse-MoreauGrad which is the $h$-MoreauGrad with the group-norm $h(\mathbf{v})=\eta \Vert \mathbf{v}\Vert_{2,1}$. $\eta$ is the sparsity parameter. Here, we extend the robustness guarantee of \Cref{Thm: MoreauGrad} to a general $h$-MoreauGrad.

\begin{theorem}\label{Thm: h-MoreauGrad}
    Consider the setting of \Cref{Thm: MoreauGrad} and suppose $h$ is a convex function. Then, for every $0<\rho<\frac{\sigma}{\beta}$, the following robustness guarantee will hold of the $h$-Moreau envelope of  $g_\sigma^\rho(\vw)$
        $$\|h\text{-}\mathrm{MG}^\rho[g_\sigma](\vw_1) - h\text{-}\mathrm{MG}^\rho[g_\sigma](\vw_2)\|_2 \leq \frac{\sigma}{\min\{\sigma\rho, \sigma-\rho\beta\}}\|\vw_1-\vw_2\|_2.$$
\end{theorem}
\begin{proof}
    We defer the proof of the above theorem into appendices due to the space limitation.
\end{proof}

\begin{algorithm}
\caption{MoreauPruner Algorithm}
\label{alg:MoreauPruner}
\begin{algorithmic}[1]
\REQUIRE samples $\vx$, network with parameter $f(\vw)$, regularization parameter $\rho$, group-sparsity $\eta$, noise std $\sigma$, stepsize $\gamma$, optimization length $T$
\STATE Initialize $\vw^{(0)} = \vw$
\FOR {$t = 0, \ldots, T$}
    \STATE draw noise vectors $\vz_1, \ldots, \vz_m \sim \mathcal{N}(0, \sigma^2 I_{d \times d})$
    \STATE compute $\vg_t =  \frac{1}{m}\sum_{i=1}^m\nabla f(\vw^{(t)}+\vz_i, \vx)$
    \STATE Update $\vw^{(t+1)} \leftarrow (1-\frac{\gamma}{\rho})\vw^{(t)} - \gamma(\vg_t - \frac{1}{\rho}\vw)$
    \IF {GroupSparse}
        \STATE Update $\vw^{(t+1)} \leftarrow \text{GroupSoftThreshold}_{\gamma \eta}(\vw^{(t+1)} - \vw) + \vw$
    \ENDIF
\ENDFOR
\STATE Compute Importance score $I(\vw) = \frac{1}{\rho}(\vw^{(T)} - \vw)\vw$
\STATE Prune network $f(\vw)$ according to $I(\vw)$
\STATE Finetune pruned network
\STATE \textbf{Return} finetuned network
\end{algorithmic}
\end{algorithm}

\subsection{Leveraging MoreauGrad for Robust Pruning}

\minisection{MoreauPruner.} With the defined MoreauGrad, we established a robust estimation on the influence of removing a weight element $w^{(k)}$ based on \Cref{eq:is},
\begin{equation}
    \mathrm{MG}^\rho\text{-}I(w^{(k)}) = |\mathrm{MG}^\rho[g](\vw)\odot\vw|^{(k)},
\end{equation}
where $g(\vw)$ is the exception of loss function $L(\vw, \mathcal{D})$ defined in \Cref{eq:loss}, $|\cdot|$ is an element-wise absolute value function and $\odot$ denotes Hadamard product. To compute, the optimal solution $\Tilde{\vw}^*_\rho(\vw)$ of the optimization problem of Moreau Envelope $g^\rho(\vw)$ can be obtained over a calibration dataset with the first-order gradient descent optimization method. Since the difference $\Tilde{\vw}^*_\rho(\vw)-\vw$ is aligned with $g^\rho$'s gradient\cite{moreau1965proximite,zhang2023moreaugrad}, we have,
\begin{equation}
    \nabla g^\rho(\vw) = -\frac{1}{\rho}(\Tilde{\vw}^*_\rho(\vw)-\vw).
\end{equation}

Then, the robust importance score of the structure $\vw_i$ and group $\mathcal{G}$ can be also estimated by \Cref{eq:structure} and \Cref{eq:group}. We denote the structural pruning method removing groups with smaller robust importance score as MoreauPruner.

\minisection{MoreauPruner-GS.} Similar to MoreauPruner, the robust estimation on the influence of removing parameter $\vw^k$ is,
\begin{equation}
    h\text{-}\mathrm{MG}^\rho\text{-}I(w^{(k)}) = |h\text{-}\mathrm{MG}^\rho[g_\sigma](\vw)\odot\vw|^{(k)},
\end{equation}
where the group sparsity is conducted at the channel level in our implementation, \textit{i.e.,} each variable subset in the $2,1$-group-norm is a channel in the model. To compute the GroupSparse-MoreauGrad, we utilize the proximal gradient descent algorithm as described in \Cref{alg:MoreauPruner}. Note that we apply the group-soft-thresholding function as the proximal operator for the $\ell_{2,1}$-norm function present in GroupSparse-MoreauGrad,
\[
\mathrm{GST}_\alpha(\vv)_{S_i} :=
\begin{cases} 
0 & \text{if } \|\vv_{S_i}\|_2 \leq \alpha \\
\left(1 - \frac{\alpha}{\|\vv_{S_i}\|_2}\right) \vv_{S_i} & \text{if } \|\vv_{S_i}\|_2 > \alpha.
\end{cases}
\]
Once the optimization of $h$-Moreau envelope ends, the GroupSparse-MoreauGrad can be calculated according to \Cref{def: h-moreaugrad}. We treat the method as MoreauPruner-GS to mark the group sparsity of importance score obtained during optimization.

%% file: docs/exp.tex
\section{Numerical Results}

In this section, we conducted experiments on several famous LLMs to evaluate the proposed MoreauPruner's performance and support our theoretical claim. We also provided further insights in the discussion subsection on how and why MoreauPruner works well.

\begin{table}[t!]
    \centering
    \caption{Algorithms' robustness against weight perturbation. \textbf{Diff} denotes the absolute difference between weight formats, bfloat16 and float16. Rounding results in changes to the last digit.}
    \subfloat[WikiText2]{
     \resizebox{0.45\textwidth}{!}{
        \setlength{\tabcolsep}{4pt}
  \begin{tabular}{c|c|cccc}
    \toprule
    \multirow{2}{*}{Method}&\multirow{2}{*}{Format}& \multicolumn{4}{c}{Pruning Ratio} \\
    &&5\%&10\%&15\%&20\%\\
    \midrule
    \multirow{3}{*}{LLM-Pruner\citep{ma2023llmpruner}}&  BF16&13.80&17.73&32.60&95.82  \\
    &FP16&13.75 &17.79 &32.10 &96.57  \\
    &Diff($\downarrow$)&\textbf{0.05} &0.07 &0.51 &0.75  \\
    \midrule
    \multirow{3}{*}{MoreauPruner}&  BF16&13.89&17.42&31.05&91.79 \\
    &FP16&13.83 &17.45 &30.99 &91.79   \\
    &Diff($\downarrow$)&\textbf{0.05} &\textbf{0.03} &\textbf{0.06} &\textbf{0.00}  \\
    \bottomrule
  \end{tabular}
    }
    }
    \hfill 
    \subfloat[PTB]{
     \resizebox{0.45\textwidth}{!}{
        \setlength{\tabcolsep}{4pt}
  \begin{tabular}{c|c|cccc}
    \toprule
    \multirow{2}{*}{Method}&\multirow{2}{*}{Format}& \multicolumn{4}{c}{Pruning Ratio} \\
    &&5\%&10\%&15\%&20\%\\
    \midrule
    \multirow{3}{*}{LLM-Pruner\citep{ma2023llmpruner}}&  BF16&25.00 &32.10 &61.87 &202.86   \\
    &FP16&24.85 &32.16 &61.15 &210.12   \\
    &Diff($\downarrow$)&0.15 &\textbf{0.06} &0.72 &7.26   \\
    \midrule
    \multirow{3}{*}{MoreauPruner}&  BF16&25.00 &32.22 &60.43&176.24  \\
    &FP16&24.95 &32.29 &60.43 &174.19  \\
    &Diff($\downarrow$)&\textbf{0.05} &\textbf{0.06} &\textbf{0.00} &\textbf{2.05}  \\
    \bottomrule
  \end{tabular}
    }
    }
    \label{tab:ppl}
    
\end{table}

\begin{table}[t!]

  \centering
  \caption{Zero-shot performance of the Pruned LLaMA-7B. $^\dagger$ denotes results from \cite{ma2023llmpruner}, $^\ddagger$ denotes results from \cite{zhang2023pruning} and $^*$ is implemented according to open-source code. The \textbf{best} results is bold. The methods proposed in this paper are filled with \textcolor{green}{green}.}
  \resizebox{0.96\textwidth}{!}{
  \setlength{\tabcolsep}{4pt}
  \begin{tabular}{c|c|cc|cccccccc}
    \toprule
    Pruning Ratio & Method& WikiText2($\downarrow$) & PTB($\downarrow$) & BoolQ & PIQA & HellaSwag & WinoGrande & ARC-e & ARC-c & OBQA & Average \\
    \midrule
    Ratio = 0\% & LLaMA-7B$^\dagger$&  12.62 & 22.14 & 73.18& 78.35&	72.99&	67.01&	67.45&	41.38	&42.40&	63.25 \\
    \midrule
    \multirow{8}{*}{\shortstack[c]{ Ratio = 20\% \\ w/o tune }}&Magnitude$^\dagger$&582.41 &1022.17&59.66&58.00&37.04&52.41&33.12&28.58&29.80&42.65  \\
    & Random$^\dagger$ & 27.51 & 43.19 & 61.83 & 71.33 & 56.26 &54.46&57.07 & 32.85&35.00&52.69 \\

    &WANDA$^\ddagger$\citep{sun2023simple}& 22.12 &38.19 &\textbf{64.93} &70.14&58.12& 55.39 &56.63& 33.98 &35.43 &53.23\\
    &LLM-Pruner$^*$\citep{ma2023llmpruner}& 19.09&	34.23 &56.91 &	75.08 &	66.81 &	60.06 &	60.94 &	36.43 &	40.00 &	56.60 \\
    &LoRAPrune$^\ddagger$\citep{zhang2023pruning}&20.67 &34.12& 57.98 &75.11 &65.81 &59.90& 62.14 &34.59 &39.98 &56.50\\
    
    &\cellcolor{green!10}SmoothGrad  & 18.91	&34.30 &59.60	&75.14	&65.98&	61.01	&60.77&	37.12	&39.80	&57.06 \\
    
    &\cellcolor{green!10}MoreauPruner &\textbf{18.61} &	\textbf{32.92} & 55.44 &	\textbf{ 76.17} &	66.47 	&\textbf{63.61} &	\textbf{61.53} &	\textbf{37.80} &	40.60&\textbf{57.37}  \\
    
    &\cellcolor{green!10}MoreauPruner-GS&18.72 &	34.91 &	62.51 &	75.52 	&\textbf{68.29} &	62.75 &	54.88 &	36.35 	&\textbf{40.80} &	57.30  \\

    \midrule
    
    \multirow{6}{*}{\shortstack[c]{ Ratio = 20\% \\ w tune }}
    &WANDA$^\ddagger$\citep{sun2023simple}& 18.43& 33.16& 65.75 &74.70 &64.52 &59.35& 60.65& 36.26& 39.40 &57.23\\
    &LLM-Pruner$^*$\citep{ma2023llmpruner} &  17.62&	30.57&	65.78&	76.44&	68.67&	64.33	&63.26&	36.35&	41.00&	59.40 \\
    &LoRAPrune$^\ddagger$\citep{zhang2023pruning}&16.80& \textbf{28.75} &65.62& \textbf{79.31} &\textbf{70.00} &62.76 &\textbf{65.87}& 37.69 &39.14& 60.05 \\
    &\cellcolor{green!10}SmoothGrad  &17.45&	30.57&	66.48	&76.99&	68.64	&65.35	&63.68	&37.80&	41.00	&59.99 \\
    &\cellcolor{green!10}MoreauPruner&17.01 	&30.27 &	66.61 &	77.04 &	68.32 &	\textbf{65.59 }&65.57 	&\textbf{38.40} &	\textbf{41.20} &	60.39  \\
    &\cellcolor{green!10}MoreauPruner-GS&\textbf{16.65} &	30.69 &	\textbf{68.87 }&	77.26 	&69.81& 	65.04 &	63.64 	&38.23 &	40.60 &	\textbf{60.49  }\\

    \bottomrule
  \end{tabular}
  }
  \label{tab:l7b}
\end{table}

\subsection{Experimental Settings}

\minisection{Pre-trained Models} To demonstrate the versatility of MoreauPruner across different scales, we evaluate it on three open-source large language models: LLaMA-7B\cite{touvron2023llama}, LLaMA-13B\cite{touvron2023llama}, and Vicuna-7B\cite{vicuna2023}.

\minisection{Evaluation} Building on prior research \cite{zhang2023pruning, ma2023llmpruner, sun2023simple}, we assess our method using seven zero-shot classification tasks on datasets centered around common sense reasoning: BoolQ \cite{clark2019boolq}, PIQA \cite{bisk2020piqa}, HellaSwag \cite{zellers2019hellaswag}, WinoGrande \cite{sakaguchi2021winogrande}, ARC-easy \cite{clark2018think}, ARC-challenge \cite{clark2018think}, and OpenbookQA \cite{mihaylov2018can}. Consistent with \cite{eval-harness}, the model ranks options in multiple-choice tasks or generates answers for open-ended questions. Furthermore, we perform a zero-shot perplexity (PPL) analysis on WikiText2 \cite{merity2016pointer} and PTB \cite{marcus1993building} with 128-token segments, aligning our methodology with that of \cite{zhang2023pruning, ma2023llmpruner}.

\minisection{Implementation Details} Aligning with the protocols of the closely related gradient-based method \cite{ma2023llmpruner}, our model pruning utilizes a calibration set of ten randomly selected, 128-token truncated sentences from the Bookcorpus \cite{zhu2015aligning}. The gradient of the Moreau envelope is computed using this calibration set, with the optimization step length fixed at ten. The pruning process typically completes in approximately 30 minutes on CPUs. In the post-training phase, a refined version of the Alpaca\cite{alpaca} dataset comprising about 50,000 samples is employed, with training extending over two epochs and generally taking three hours on a single NVIDIA RTX 3090 Ti GPU for 7B models. Detailed hyper-parameter selections are available in the Appendices.

\minisection{Structural Pruning Baselines}
We compare MoreauPruner against two fundamental pruning techniques: \textit{Magnitude} and \textit{Random}. Magnitude pruning evaluates weight significance based on the magnitude of the weight matrix, whereas Random pruning indiscriminately removes weights. Additionally, we benchmark against three advanced alternatives: \textit{LLM-Pruner} \cite{ma2023llmpruner}, which uses a gradient-based metric to determine weight importance; \textit{LoraPrune} \cite{zhang2023pruning}, which utilizes a LoRA\cite{hu2021lora}-guided pruning criterion; and \textit{WANDA} \cite{sun2023simple}, designed for unstructured or semi-structured pruning but adaptable to other structural frameworks. 

We also introduce \textit{SmoothGrad}, a preliminary version of MoreauPruner that enhances network smoothness by applying Gaussian smoothing during the inference, as we explained in \Cref{Thm: MoreauGrad}. The importance scores are estimated with the smoothed gradient using \Cref{eq:is}, and the we still remove those parameter groups with lower importance scores. A thorough comparison of these methods is documented in \Cref{tab:baseline_compare}.

\begin{table}[t!]
  \centering
  
  \caption{Zero-shot performance of the compressed Vicuna-7B.}
  \resizebox{0.96\textwidth}{!}{
  \setlength{\tabcolsep}{4pt}
  \begin{tabular}{c|c|cc|cccccccc}
    \toprule
    Pruning Ratio & Method& WikiText2($\downarrow$) & PTB($\downarrow$) & BoolQ & PIQA & HellaSwag & WinoGrande & ARC-e & ARC-c & OBQA & Average \\
    \midrule
    Ratio = 0\% & Vicuna-7B$^\dagger$ & 16.11 &	61.39 &	76.54 &	77.20 &	70.70 &	67.25 &	65.15 	&41.30 &	40.80 &	62.71  \\
    \midrule
    \multirow{6}{*}{\shortstack[c]{ Ratio = 20\% \\ w/o tune }}&Magnitude$^\dagger$& 3539.98&5882.21&55.90&56.15&32.37&51.85&30.01&28.41&28.20&40.41\\
    & Random$^\dagger$ &34.63&112.44& 61.47&70.89&54.67&56.27&55.60&31.74&34.60&52.18\\
    
    &LLM-Pruner$^*$\citep{ma2023llmpruner}&25.74&\textbf{92.87}&\textbf{61.62}&74.76&63.76&56.20&63.22&36.69&37.00&56.18   \\
    &\cellcolor{green!10}SmoothGrad  &  25.99&\textbf{92.87}&60.73&74.97&63.75&54.22&64.90&37.03&37.60&56.17  \\
    &\cellcolor{green!10}MoreauPruner & \textbf{25.54}&94.34&56.82&\textbf{75.79}&64.73&56.35&\textbf{65.95}&\textbf{37.88}&\textbf{39.80}&\textbf{56.76} \\
    &\cellcolor{green!10}MoreauPruner-GS& 30.69&108.16&61.47&75.24&\textbf{66.56}&\textbf{61.72}&57.24&37.12&38.00&\textbf{56.76} \\

    \midrule
    
    \multirow{4}{*}{\shortstack[c]{ Ratio = 20\% \\ w tune }}
    & LLM-Pruner$^*$\citep{ma2023llmpruner} & 19.47&72.33&64.43&76.44&65.39&60.46&63.22&35.92&38.20&57.72  \\
    &\cellcolor{green!10}SmoothGrad  & 19.51&\textbf{72.05}&63.46&75.68&65.38&60.93&62.79&36.43&38.80&57.64  \\
    &\cellcolor{green!10}MoreauPruner& 19.66&73.47&63.15&76.77&65.96&60.85&65.74&37.12&\textbf{40.60}&58.60 \\
    &\cellcolor{green!10}MoreauPruner-GS&\textbf{19.13}&73.76&\textbf{65.41}&\textbf{76.99}&\textbf{68.17}&\textbf{65.27}&\textbf{66.37}&\textbf{38.23}&39.80&\textbf{60.03 }\\

    \bottomrule
  \end{tabular}
  }
  \label{tab:v7b}
 
\end{table}

\begin{table}[t!]
  \centering
  
  \caption{Zero-shot performance of the compressed LLaMA-13B.}
  \resizebox{0.96\textwidth}{!}{
  \setlength{\tabcolsep}{4pt}
  \begin{tabular}{c|c|cc|cccccccc}
    \toprule
    Pruning Ratio & Method& WikiText2($\downarrow$) & PTB($\downarrow$) & BoolQ & PIQA & HellaSwag & WinoGrande & ARC-e & ARC-c & OBQA & Average \\
    \midrule
    Ratio = 0\% & LLaMA-13B$^\dagger$&  11.58&44.54&68.47&78.89&76.24&70.09&74.58&44.54&42.00&64.97  \\
    \midrule
    \multirow{4}{*}{\shortstack[c]{ Ratio = 20\% \\ w/o tune }} 
    & LLM-Pruner$^*$\citep{ma2023llmpruner}& \textbf{16.43}&\textbf{59.96}&63.00&77.53&73.79&64.33&\textbf{69.07}&\textbf{40.96}&40.60&61.33  \\
    &\cellcolor{green!10}SmoothGrad  & 16.55&\textbf{59.96}&62.94&77.04&73.78&65.98&68.35&40.53&\textbf{41.40}&61.43 \\
    &\cellcolor{green!10}MoreauPruner &16.95&61.39&62.48&\textbf{77.64}&73.61&\textbf{66.38}&67.47&39.68&40.60&61.12   \\
    &\cellcolor{green!10}MoreauPruner-GS&17.11&61.39&\textbf{72.97}&77.53&\textbf{74.44}&64.09&66.08&40.44&\textbf{41.40}&\textbf{62.42 }\\

    \midrule
    
    \multirow{4}{*}{\shortstack[c]{ Ratio = 20\% \\ w tune }}
    & LLM-Pruner$^*$\citep{ma2023llmpruner} & 15.04&57.00&67.28&79.00&\textbf{75.13}&69.06&\textbf{71.68}&41.89&43.60&63.95 \\
    &\cellcolor{green!10}SmoothGrad   &\textbf{15.01}&\textbf{56.55}&66.39&79.05&74.95&69.46&71.17&42.75&43.40&63.88 \\
    &\cellcolor{green!10}MoreauPruner& 15.52&57.44&64.86&\textbf{79.22}&75.07&\textbf{70.48}&\textbf{71.68}&\textbf{43.60}&42.80&63.96  \\
    &\cellcolor{green!10}MoreauPruner-GS&15.28&57.67&\textbf{75.17}&78.24&74.77&68.19&70.12&43.09&\textbf{45.00}&\textbf{64.94} \\

    \bottomrule
  \end{tabular}
  }
  \label{tab:l13b}
\end{table}

\subsection{Robustness Against Weight Perturbation}

As we previously discussed, few-shot gradient-based pruning methods are significantly influenced by the changes of the gradient. Even minor differences between FP16 and BF16 can lead to markedly different pruning outcomes. In contrast, MoreauGrad is theoretically robust against norm-bounded weight perturbation. To validate this assertion, we adhered to the channel-wise pruning protocol established in prior research \cite{ma2023llmpruner}, removing a fixed ratio of channels based on their importance score. 

Considering that LLMs are usually trained on BF16 but inferred on FP16, our experiments were conducted on the LLaMA-7B model using both BF16 and FP16 weight bit formats. We standardized the calibration sample selection during pruning across different settings to ensure a fair comparison. Upon completing the pruning process, we performed a zero-shot perplexity (PPL) analysis using 128-token segments on the WikiText2 and PTB datasets and compared the discrepancies between FP16 and BF16. The findings are presented in \Cref{tab:ppl}.

The results indicate that, for MoreauPruner, the performance under different weight format is closer to each other, which indicates a better consistency of pruning outcomes. The result demonstrates the robustness of MoreauPruner against weight perturbations caused by different weight formats.

\subsection{Zero-shot Performance}

We have developed two variants of our method, named \textit{MoreauPruner} and \textit{MoreauPruner-GS}, according to whether the sparsity penalty is applied. These techniques were tested on three pretrained models: LLaMA-7B, Vicuna-7B, and LLaMA-13B, with their performance detailed in \Cref{tab:l7b,tab:v7b,tab:l13b}.

The evaluations indicate that MoreauPruner can effectively maintain the model performance. For example, with a 20\% reduction in parameters on LLaMA-7B, \textit{MoreauPruner} and \textit{MoreauPruner-GS} maintains 95.48\% and 95.64\% of the original performance with a quick post-training. 
For Vicuna-7B and LLaMA-13B, \textit{MoreauPruner-GS} surpassed \textit{MoreauPruner} by a notable margin, thanks to the structural sparsity introduced during optimization. On Vicuna-7B, \textit{MoreauPruner-GS} maintains 96.16\% of the original performance. For the larger model LLaMA-13B, We have noticed that the performance gap between the compressed and original models is closer than that of the 7B models. After a quick recovery, the zero-shot accuracy of the compressed with 80\% parameters is nearly equivalent to the original model's performance (64.94\% vs. 64.97\%). Such a phenomenon may indicate more redundant weights in the larger models. In other words, those huge LLMs ($\geq$13B) can be potentially inferred without a trade-off on performance.

\subsection{Further Discussion}
\label{sec:5.4}

In this subsection, we extended our experiment to identify how \textit{MoreauPruner} works. We also discussed that with more computational resource, how can \textit{MoreauPruner} be further improved.

\minisection{Effect of Function Smoothing} In our preliminary evaluations, we introduced \textit{SmoothGrad} to assess the impact of function smoothing. This approach often matches or exceeds the performance of gradient-based competitors. Notably, on the benchmark model LLaMA-7B, \textit{SmoothGrad} outperformed all baseline methods prior to finetuning. These findings suggest that gradient-based pruning methods could benefit from function smoothing, as it helps mitigate the excessive sharpness of certain parameters within the differential space.

\begin{table}[t!]
  \centering
  
  \caption{The effect of larger recovery set.}
  \resizebox{0.96\textwidth}{!}{
  \setlength{\tabcolsep}{4pt}
  \begin{tabular}{c|c|c|cccccccc}
    \toprule
    Pruning Ratio & Method& Recovery Set&BoolQ & PIQA & HellaSwag & WinoGrande & ARC-e & ARC-c & OBQA & Average \\
    \midrule
    Ratio = 0\% & LLaMA-7B$^\dagger$& N/A& 73.18& 78.35&	72.99&	67.01&	67.45&	41.38	&42.40&	63.25 \\
    \midrule
     \multirow{2}{*}{\shortstack[c]{ Ratio = 20\%}}&MoreauPruner-GS&50k\cite{alpaca}	&68.87 &	77.26 	&69.81& 	65.04 &	63.64 	&38.23 &	40.60 &	60.49  \\
        &MoreauPruner-GS&2.59M\cite{wu2023lamini}&76.97 &	76.82 &	68.51 	&66.30 &	70.88 	&41.89 &	40.80 &	63.17 \\

    \bottomrule
  \end{tabular}
  }
  \label{tab:largerecover}
\end{table}
\begin{table}[t!]
  \centering
  \caption{The performance of the MoreauPruner-GS on LLaMA-7B with different calibration set size.}
  \resizebox{0.480\textwidth}{!}{
  \setlength{\tabcolsep}{4pt}
\begin{tabular}{c|c|ccc}
    \toprule
    Pruning Ratio & Calibration& WikiText2($\downarrow$) & PTB($\downarrow$) & Average \\
    \midrule
    \multirow{2}{*}{\shortstack[c]{ Ratio = 20\% \\ w/o tune }}&10&18.72&34.91&57.30\\
    &1000&\textbf{18.50} 	&\textbf{32.16}& \textbf{58.12}\\
    
    \midrule
    \multirow{2}{*}{\shortstack[c]{ Ratio = 20\% \\ w tune }}&10&\textbf{16.65}&30.69&60.49\\
    &1000&16.95 	&\textbf{30.21}& \textbf{60.63}\\
    \bottomrule
\end{tabular}
}
\end{table}

\minisection{Larger Recovery Set.} In the primary section of the results, the recovery phase was conducted on Alpaca\cite{alpaca}, utilizing a dataset of 50,000 samples. To demonstrate the potential enhancement achieved by the pruned model, we carried out an experiment on a significantly larger dataset\cite{wu2023lamini}, consisting of 2.59 million samples. The findings, presented in \Cref{tab:largerecover}, reveal that the performance of the compressed model closely approximates that of the base model (63.17\% v.s. 63.25\%), respectively. These results further substantiate the hypothesis of the presence of redundant weights in LLMs.

\minisection{Larger Calibration Set.} We enlarge the size of calibration set utilized during pruning phase. We found a larger calibration set can efficiently improving the pruning quality. Estimating gradient importance on 1000 samples raise the average zero-shot accuracy from 57.30\% to 58.12\% and decrease PPL by 0.22 and 2.75 on WikiText2 and PTB. However, the difference on post-finetuning performance is shrinking, resulting as only 0.14\% difference on average accuracy. 

\minisection{Results on LLaMA3-8B.} We extend our result on LLaMA3-8B\cite{llama3modelcard}, a stronger foundational model that is pre-trained with more high-quality data compared with previous version. The result can be found in \Cref{tab:llama3}. MoreauPruner-GS still works well on stronger LLaMA3-8B without hyper-paramerter modification. We have noticed that the performance of pruned LLaMA3-8B drops more than that of LLaMA-7B. This may lead by the fact that the pretraining of LLaMA3-8B is more sufficient according to official report and there is less redundant model weight. However, the pruned LLaMA3-8B still beats the original LLaMA-7B by a noticeable margin (63.25\% vs. 65.37\%).

\minisection{Effect of Pruning Ratio} We explored the influence of varying pruning ratios as illustrated in \Cref{fig:ratio}. It is evident that our methods consistently work well across different pruning ratios. This stability underscores the robustness and effectiveness of our pruning strategies.

\minisection{Impact of Hyper-parameters.} The hyper-parameter $\eta$ controls the ratio of group-sparsity of MoreauPruner-GS during optimization. We conduct an ablation study on LLaMA-7B with 20\% sparsity to evaluate the impact of different hyper-parameter values $\eta$. The results illustrated in \Cref{fig:eta} give the average 0-shot accuracy after finetuning. According to the results, we choose $\eta$=5e-6 for all the experiments in this paper.

\begin{table}[t!]
  \centering
  
  \caption{Zero-shot performance of the compressed LLaMA3-8B.}
  \resizebox{0.96\textwidth}{!}{
  \setlength{\tabcolsep}{4pt}
  \begin{tabular}{c|c|cc|cccccccc}
    \toprule
    Pruning Ratio & Method& WikiText2($\downarrow$) & PTB($\downarrow$) & BoolQ & PIQA & HellaSwag & WinoGrande & ARC-e & ARC-c & OBQA & Average \\
    \midrule
    Ratio = 0\% & LLaMA-8B$^\dagger$&14.14 &	27.98&  81.35 	&80.79& 	79.17 &	72.53 &	80.09 &	53.41 &	45.00 &	70.33   \\
    \midrule
     \multirow{2}{*}{\shortstack[c]{ Ratio = 20\% \\ w/o tune }}&LLM-Pruner\cite{ma2023llmpruner}&25.74 &	45.69 	&67.55 	&74.97 &	63.33 	&67.80& 	62.29 &	35.49 &	36.60 &	58.29    \\
        &MoreauPruner-GS&\textbf{25.40 }&\textbf{	43.78 	}&\textbf{73.73 }&\textbf{	75.08 }&\textbf{	64.93 	}&\textbf{68.03 	}&\textbf{66.11 }&\textbf{	39.25 }&\textbf{	37.60 }&\textbf{	60.68}  \\
    \midrule
     \multirow{2}{*}{\shortstack[c]{ Ratio = 20\% \\ w tune }}&LLM-Pruner\cite{ma2023llmpruner}&23.71& 	42.01 	&\textbf{77.52 }&	77.69 	&71.75 	&67.96 	&71.63 &	42.24& 	40.00 	&64.11    \\
        &MoreauPruner-GS&\textbf{22.98 }&\textbf{	39.25 	}&76.57&\textbf{ 	78.67 }&\textbf{	73.17 
        }&\textbf{	69.14 	}&\textbf{74.49 	}&\textbf{43.77 }&\textbf{	41.80 }&\textbf{	65.37  } \\

    \bottomrule
  \end{tabular}
  }
  \label{tab:llama3}
\end{table}

\begin{figure}[t!]
    \centering
    
    \begin{minipage}[t]{0.480\textwidth}
    \centering
    \includegraphics[height=4.0cm]{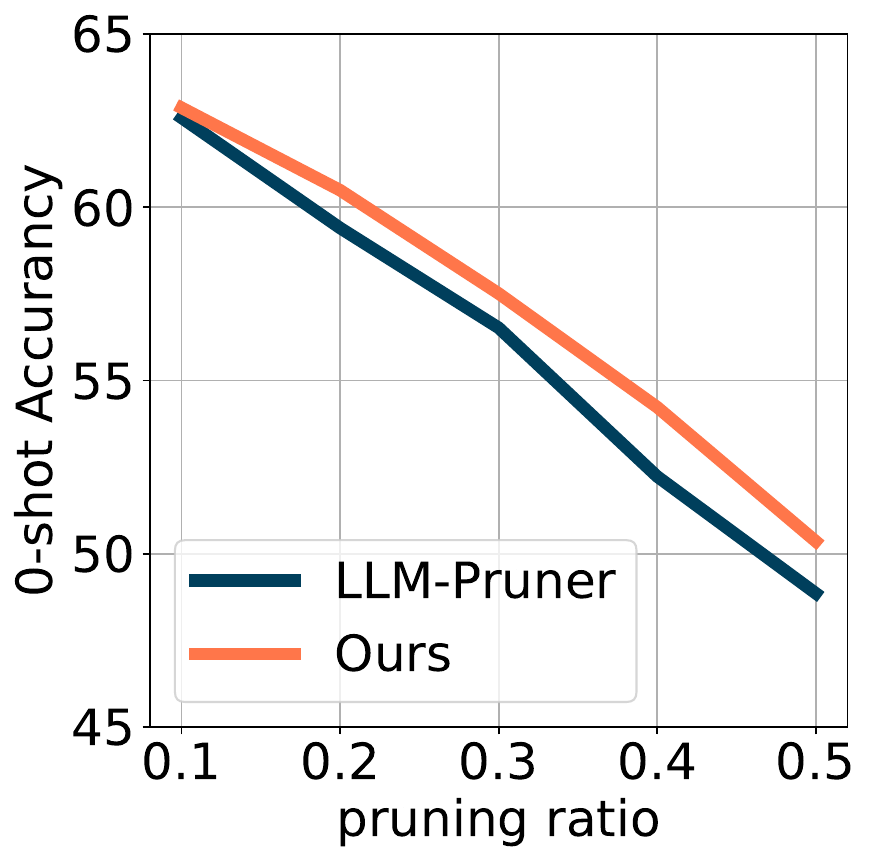}
    \caption{The effect of different pruning ratio.}
    \label{fig:ratio}
    \end{minipage}
    \hfill
    \begin{minipage}[t]{0.480\textwidth}
    \centering
    \includegraphics[height=4.0cm]{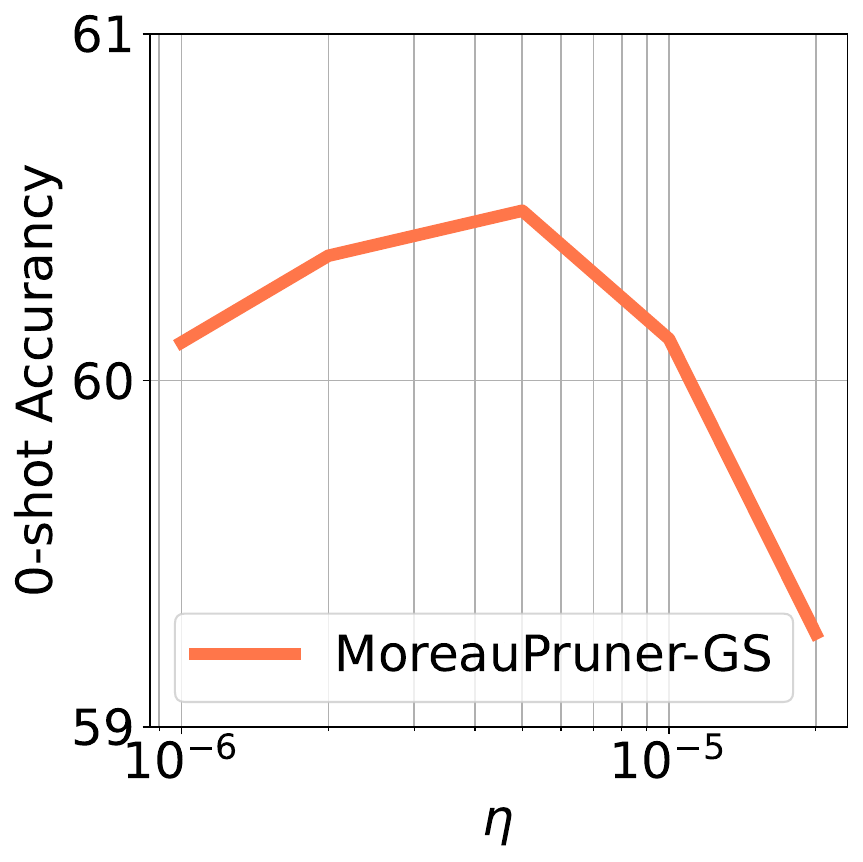}
    \caption{Ablation study on hyper-paramter $\eta$.}
    \label{fig:eta}
    \end{minipage}
    
\end{figure}

%% file: docs/conclusion.tex
\section{Conclusion}

In this paper, we discussed how minor changes in model weights can lead to unstable pruning results for large language models (LLMs). To address this instability, we introduced MoreauPruner, a weight-perturbation structural pruning method. Our theoretical analysis demonstrates that MoreauPruner is robust to norm-bounded perturbations. Numerical experiments conducted on well-known LLMs, such as LLaMA-7B, LLaMA-13B, LLaMA-3-8B, and Vicuna-7B, suggest that MoreauPruner can efficiently compress LLMs while maintaining their performance. For future work, we propose combining structural pruning technology with other model compression methods to accelerate model inference and reduce computational costs.

\minisection{Limitations.} The authors acknowledge that the number of parameters utilized in the models for this paper only reach 13B due to limited hardware budget. The performance of MoreauPruner on extremely large-scale models (e.g., 30B, 70B, etc.) will be further explored once enough hardware resources are available.

%% file: docs/proof.tex
\section{Proof}

\subsection{Proof of \Cref{Thm: MoreauGrad}}

As Theorem~\Cref{Thm: MoreauGrad} assumes a Lipschiz function $g$, 
we can apply the Stein's lemma \cite{landsman2013note} to show
\begin{align*}
&\nabla g_\sigma(\vw) \, =\, \mathbb{E}[\nabla g(\vw + \rmZ)] \,=\, \mathbb{E}[g(\vw+\rmZ)\frac{\rmZ}{\sigma^2}]
\end{align*}

Therefore, for every $\vw,\vw'\in\mathbb{R}^d$ and unit-$\ell_2$-norm vector $\Vert \vu \Vert_2=1$ we have the following

\begin{equation}
    \begin{aligned}
        \left|\vu^\top(\nabla g_\sigma(\vw)-\nabla g_\sigma(\vw'))\right| =& \left|\vu^\top(\mathbb{E}[\nabla g(\vw+\rmZ)] - \mathbb{E}[\nabla g(\vw'+\rmZ)])\right| \\
    =& \left|\vu^\top(\mathbb{E}\left[\frac{\rmZ}{\sigma^2}g(\vw+\rmZ)\right]-\mathbb{E}\left[\frac{\rmZ}{\sigma^2}g(\vw'+\rmZ)\right])\right| \\
    =& \left|\mathbb{E}\left[\frac{\vu^\top\rmZ}{\sigma^2}(g(\vw+\rmZ)-g(\vw'+\rmZ))\right]\right| \\
    \leq& \mathbb{E}\left[\frac{\left|\vu^\top\rmZ\right|}{\sigma^2}\left|g(\vw+\rmZ)-g(\vw'+\rmZ)\right|\right] \\
    \leq& \mathbb{E}\left[\frac{\left|\vu^\top\rmZ\right|}{\sigma^2}\beta\Vert\vw-\vw'\Vert_2\right] \\
    =&\frac{\beta\Vert\vw-\vw'\Vert_2}{\sigma}\mathbb{E}\left[\bigl\vert\frac{ \vu^\top\rmZ}{\sigma}\bigr\vert\right]  \\
    \leq& \frac{\beta\Vert\vw-\vw'\Vert_2}{\sigma}.
    \nonumber
    \end{aligned}
\end{equation}

In the above, note that $ \frac{\mathbf{u}^\top \mathbf{Z}}{\sigma}\sim \mathcal{N}(0,1)$. As a result, the gradient of $g_\sigma$ will be $\frac{\beta}{\sigma}$-Lipschitz, and $g_\sigma$ is $\frac{\beta}{\sigma}$-smooth, which means for every $\vw, \vw'$ we have,
$$
\left|g_\sigma(\vw') - \nabla g_\sigma(\vw)^\top(\vw'-\vw)\right|\leq\frac{\beta}{2\sigma}\Vert\vw-\vw'\Vert_2^2
$$

As a result,  
$ \Theta(\vw) = g_\sigma(\vw) + \frac{\beta}{2\sigma}\Vert\vw\Vert_2^2$ will be a convex function. 
Therefore, we can rewrite the definition of the Moreau envelope as
\begin{equation}
    \begin{aligned}
        g_\sigma^\rho(\vw) =& \min_{\Tilde{\vw}\in\mathbb{R}^d}\Theta(\Tilde{\vw}) -\frac{\beta}{2\sigma}\Vert\Tilde{\vw}\Vert_2^2 + \frac{1}{2\rho}\Vert\Tilde{\vw}-\vw\Vert_2^2 \\
        =& \min_{\Tilde{\vw}\in\mathbb{R}^d}\Theta(\Tilde{\vw}) +(\frac{1}{2\rho}-\frac{\beta}{2\sigma})\Vert\Tilde{\vw}\Vert_2^2 - \frac{1}{\rho}\vw^\top\Tilde{\vw} + \frac{1}{2\rho}\Vert\vw\Vert_2^2\\
        =& \frac{1}{2\rho}\Vert\vw\Vert_2^2 - \frac{1}{\rho}\max_{\Tilde{\vw}\in\mathbb{R}^d}\left\{ \vw^\top\Tilde{\vw} - \rho \Theta(\Tilde{\vw}) - \frac{\sigma-\rho\beta}{2\sigma}\Vert\Tilde{\vw}\Vert_2^2\right\}.\nonumber
    \end{aligned}
\end{equation}

Therefore, $\rho g_\sigma^\rho(\vw)$ is the subtraction of the Fenchel conjugate of $\vc(\vw) = \rho\Theta(\vw) + \frac{\sigma-\rho\beta}{2\sigma}\Vert\Tilde{\vw}\Vert_2^2$ from the 1-strongly-convex $\frac{1}{2}\Vert\vw\Vert_2^2$. Then, we apply the result that the Fenchel conjugate of a $\mu$-strongly convex function is $\frac{1}{\mu}$-smooth convex function in \cite{zhou2018fenchel}. Therefore, the following Fenchel conjugate 
$$
\vc^*(\vw):=\max_{\Tilde{\vw}\in\mathbb{R}^d}\left\{ \vw^\top\Tilde{\vw} - \rho\Theta(\Tilde{\vw}) - \frac{\sigma-\rho\beta}{2\sigma}\Vert\Tilde{\vw}\Vert_2^2\right\}
$$
is a $\frac{\sigma}{\sigma-\rho\beta}$-smooth convex function. Since, we subtract two convex functions from each other where the second one has a constant Hessian $I$, then the resulting function will be smooth of the following degree:
$$
\frac{1}{\rho}\mathrm{max}\left\{|\frac{\sigma}{\sigma-\rho\beta}-1|,|0-1|\right\} = \frac{\sigma}{\min\{\sigma\rho, \sigma-\rho\beta\}},
$$
which completes the proof of the theorem.

\subsection{Proof of \Cref{Thm: h-MoreauGrad}}

To prove \Cref{Thm: h-MoreauGrad}, we note that the additional $h$ is a convex function. 
Given the formulation of the $h$-Moreau envelope of $g_\sigma^\rho(\vw)$ and the assumption $0<\rho<\frac{\sigma}{\beta}$ in the theorem, we have
\begin{equation}
    \begin{aligned}
    g_{\sigma,h}^\rho(\vw):= &\min_{\Tilde{\vw}\in\mathbb{R}^d}g_\sigma(\Tilde{\vw}) + \frac{1}{2\rho}\Vert\Tilde{\vw}-\vw\Vert_2^2 + h(\Tilde{\vw}-\vw), \\
    =& \min_{\Tilde{\vw}\in\mathbb{R}^d}\Theta(\Tilde{\vw}) +(\frac{1}{2\rho}-\frac{\beta}{2\sigma})\Vert\Tilde{\vw}\Vert_2^2 - \frac{1}{\rho}\vw^\top\Tilde{\vw} + \frac{1}{2\rho}\Vert\vw\Vert_2^2 + h(\Tilde{\vw} - \vw),
        \nonumber
    \end{aligned}
\end{equation}
where $ \Theta(\vw) = g_\sigma(\vw) + \frac{\beta}{2\sigma}\Vert\vw\Vert_2^2$ is a convex function. Then the function $\phi:\mathbb{R}^d\rightarrow\mathbb{R}$ defined as
$$
\phi(\Tilde{\vw}) = (\frac{1}{2\rho}-\frac{\beta}{2\sigma})\Vert\Tilde{\vw}\Vert_2^2 - \frac{1}{\rho}\vw^\top\Tilde{\vw}
$$
is a $\frac{\sigma-\rho\beta}{\sigma\rho}$-strongly-convex function. As a result, $\Theta(\Tilde{\vw}) + \phi(\Tilde{\vw}) + h(\Tilde{\vw} - \vw)$ is strongly-convex function with strong-convexity degree $\frac{\sigma-\rho\beta}{\sigma\rho}$. Therefore, the optimization of $h$-Moreau envelope has a unique locally and globally optimal solution. 
we define the proximal operator of $h$ function as
$$
\mathrm{prox}_{h(\cdot)}(\vw):=\argmin_{\vw'\in\mathbb{R}^d} h(\vw') + \frac{1}{2}\Vert\vw'-\vw\Vert_2^2.
$$

Then since the objective function of $h$-Moreau envelope consists of the following two convex functions (w.r.t. $\boldsymbol{\delta}:=\Tilde{\vw}-\vw$) $t_\vw(\boldsymbol{\delta}):= g_\sigma(\vw+\boldsymbol{\delta}) + \frac{1}{2\rho}\Vert\boldsymbol{\delta}\Vert_2^2$ and $h(\boldsymbol{\delta})$, the optimal solution $\boldsymbol{\delta}^*$ will satisfy the following equation with $\gamma>0$:
$$
\boldsymbol{\delta}^* = \mathrm{prox}_{\gamma h(\cdot)}\bigl(\boldsymbol{\delta}^*-\gamma\nabla t_{\vw}(\boldsymbol{\delta}^*)\bigr) \overset{\gamma=\rho}{=} \mathrm{prox}_{\rho h(\cdot)}\bigl(-\rho\nabla g_\sigma(\vw+\boldsymbol{\delta}^*)\bigr).
$$
The above implies that, if we use $\psi$ to denote the identity map we will get:
$$
\boldsymbol{\delta}^*(\vw) = \left((\psi + \mathrm{prox}_{\rho h(\cdot)}\circ \rho \nabla g_\sigma)^{-1} - \psi\right)(\vw).
$$
Note that in the above $\psi + \mathrm{prox}_{\rho h(\cdot)}\circ \rho \nabla g_\sigma$ will be a $(1-\frac{\rho\beta}{\sigma})$-monotone operator, where we call $t:\mathbb{R}^d\rightarrow\mathbb{R}^d$ $\tau$-monotone if for every $\vw, \vv\in\mathbb{R}^d$:
$$
(\vv-\vw)^\top\bigl(t(\vv)-t(\vw)\bigr) \geq \tau\Vert\vv-\vw\Vert_2^2.
$$
The monotonicity arises because the gradient of a $\lambda$-weakly convex function is -$\lambda$-monotone, and the proximal operator is known to be 1-monotone. Hence, $\boldsymbol{\delta}^*(\vw)$ will be a Lipschitz function with the following Lipschitz constant (note that $(\psi + \mathrm{prox}_{\rho h(\cdot)}\circ \rho \nabla g_\sigma)^{-1}$ is a monotone function with a degree between 0 and $\frac{\sigma}{\sigma-\rho\beta}$):
$$
\mathrm{max}\left\{|\frac{\sigma}{\sigma - \rho\beta} -1|,|0-1|\right\} = \mathrm{max}\left\{\frac{\rho\beta}{\sigma-\rho\beta},1\right\}.
$$

Therefore, for any given convex function $h$, the $h$-MoreauGrad
$$
h\text{-}\mathrm{MG}^\rho[g](\mathbf{w}) := \frac{1}{\rho}\boldsymbol{\delta}^*(\vw)
$$
will be a Lipschitz function with the constant $\frac{\sigma}{\mathrm{min}\left\{\sigma\rho, \sigma - \rho\beta\right\}}$.
Then the proof the theorem is finished.

%% file: docs/appendix.tex
\section{Experiment Details \& Extra Results}
\subsection{A detailed comparison of methods}

\begin{table}[th]
  
  \centering
  
  \caption{A detailed comparison between methods.}
  \renewcommand{\arraystretch}{1.2}
  \resizebox{0.98\textwidth}{!}{
  \setlength{\tabcolsep}{4pt}
  \begin{tabular}{cccccc}
    \toprule
    Method&Pruning Criterion& Calibration Set (Size) & Post-Training Set (Size) & Iteratively Pruning &Smoothness\\
    \midrule
    Random & random &N/A&N/A&\ding{55} &\ding{55}\\
    Magnitude& $\|\vw_i\|_2$&N/A&N/A&\ding{55}&\ding{55}\\
    
    WANDA\citep{sun2023simple}& $|\vw^k|\|x_i\|_2$ &C4(0.128k)& C4(20k)&\ding{55}&\ding{55}\\
    
    LoRAPrune\citep{zhang2023pruning}& $(\text{LoRA-guided } \frac{\partial L(\vw,\mathcal{D})}{\partial\vw^k})\vw^k$&C4(20k)&C4(20k)&\ding{51}&\ding{55}\\
    LLM-Pruner\citep{ma2023llmpruner}& $\frac{\partial L(\vw,\mathcal{D})}{\partial\vw^k}\vw^k$& Bookcorpus(0.01k) &  Alpaca(50k)&\ding{55}&\ding{55}\\
    \cellcolor{green!10}SmoothGrad & $\mathbb{E}\frac{\partial L(\vw+\vz,\mathcal{D})}{\partial\vw^k}\vw^k$ & Bookcorpus(0.01k) &  Alpaca(50k)&\ding{55}&\ding{51}\\
    \cellcolor{green!10}MoreauPruner & $\text{MG}_{\rho}\left(L(\vw,\mathcal{D})\right)\vw^k$ & Bookcorpus(0.01k) &  Alpaca(50k)&\ding{55}&\ding{51}\\
    \cellcolor{green!10}MoreauPruner-GS&$\text{GS-MG}_{\rho,\eta}\left(L(\vw,\mathcal{D})\right)\vw^k$& Bookcorpus(0.01k) &  Alpaca(50k)&\ding{55}&\ding{51}\\
    \bottomrule
  \end{tabular}
  }
  \label{tab:baseline_compare}
  \renewcommand{\arraystretch}{1}
\end{table}

We list the comparison on the experiment setting utilized in our baselines, which can be found in \Cref{tab:baseline_compare}. 

\subsection{Parameters Choosing}

In the pruning stage, for SmoothGrad, we randomly pick a batch from BookCorpus \cite{zhu2015aligning} with ten 128-token truncated sentences. We then pass the batch to the model 100 times. We utilized a element-wised Gaussian smoothing, \textit{i.e.,} for weight parameter $w^{(k)}$, the intensity of Gaussian is $\sigma=0.05|w^{(k)}|$, where $|\cdot|$ denotes the absolute value function. The smooth gradient is empirically calculated by averaging the gradients of each forward pass.

For both MoreauPruner and MoreauPruner-GS, we also apply the element-wise Gaussian smoothing to the model weights during the optimization of the gradient of the Moreau Envelope, as SmoothGrad does. The hyper-parameter $\rho$ is set to 0.05 for MoreauPruner and 0.2 for MoreauPruner-GS. The stepsize $\gamma$ used in the optimization of the gradient of the Moreau Envelope is 1e-3 for MoreauPruner and 2e-4 for MoreauPruner-GS. The hyper-parameter $\eta$ is set to 5e-6 as explained in the main text. We conducted a parameter search on LLaMA-7B to find suitable hyper-parameters.

In the fine-tuning stage, we use the protocol from previous work \cite{ma2023llmpruner}. The batch size is 64. The learning rate is 1e-4 for Alpaca \cite{alpaca} and 5e-5 for Lamini \cite{wu2023lamini}. The training length is two epochs for Alpaca and three epochs for Lamini.

\subsection{Extra Results}

\minisection{Compared with Scratch Training.} We compare our MoreauPruner-GS with StableLM-3B\footnote{https://huggingface.co/stabilityai/stablelm-tuned-alpha-3b} with a similar parameter size. With MoreauPruner-GS, We prune LLaMA-7B and get a compact model with 3.45B parameters. Both models are finetuned on Alpaca\cite{alpaca} dataset for a fair comparison. The result can be found in \Cref{tab:scratch}. MoreauPruner-GS sometimes achieves better results compared with LLMs that are trained from scratch. 
We also recognize that the pruned model may not consistently surpass other models with similar scale, due to the significant disparity in the size of the training corpus.

\begin{table}[t!]
  \centering
  \caption{Comparison between scratch-training and LLaMA-3B obtained by MoreauPruner-GS}
  \label{tab:scratch}
  \setlength{\tabcolsep}{4pt}
  \resizebox{0.92\textwidth}{!}{
  
  \begin{tabular}{c|c|cccccccc}
    \toprule
    Method& \#Param  & BoolQ & PIQA & HellaSwag & WinoGrande & ARC-e & ARC-c & OBQA & Average \\
    \midrule
    StableLM-3B$^\dagger$ &3.6B& 48.78&\textbf{69.48}&44.52&54.62&50.93&25.17&27.40&45.84 \\
 
    MoreauPruner-GS& 3.5B&\textbf{62.26}&68.39&\textbf{49.58}&\textbf{55.72}&\textbf{50.97}&\textbf{30.20}&\textbf{35.40}&\textbf{50.36}  \\

    \bottomrule
  \end{tabular}
  }
\end{table}